\documentclass[submission,copyright,creativecommons]{eptcs}
\providecommand{\event}{QPL 2014}
\usepackage[T1]{fontenc}
\usepackage[utf8]{inputenc}
\usepackage{amsmath}
\usepackage{cases}
\usepackage{amsthm}
\usepackage{amsfonts}
\usepackage{mathtools}
\usepackage{bbm}
\usepackage{tikz}
\usepackage{pgf}
\usepackage{graphicx}
\usepackage[caption=false]{subfig}
\usepackage{multicol}
\usepackage{todonotes}
\usepackage{float}
\usepackage{url}
\usepackage{gb4e}
\usepackage{enumerate}
\usetikzlibrary{positioning,calc,arrows}

\newtheorem{defi}{Definition}
\newtheorem{lemma}{Lemma}
\newtheorem{prop}{Proposition}
\newtheorem{thm}{Theorem}
\newtheorem{coro}{Corollary}

\usepackage{todonotes}

    \renewcommand{\bar}[1]{%
        \overline{#1}}

\newcommand{\cC}[0]{%
\mathcal{C}}

\title{Complexity of Grammar Induction for Quantum Types}
\author{Antonin Delpeuch
    \institute{École Normale Supérieure\\
        45 rue d'Ulm \\
    75005 Paris, France}
\email{antonin.delpeuch@ens.fr}}

\def\titlerunning{Complexity of Grammar Induction for Quantum Types}
\def\authorrunning{A. Delpeuch}
\begin{document}

\date{\today}
\maketitle

\begin{abstract}
    Most categorical models of meaning use a functor from the syntactic category to
    the semantic category. When semantic information is available,
    the problem of grammar induction can therefore
    be defined as finding preimages of the semantic types under this
    forgetful functor, lifting the information flow from the semantic
    level to a valid reduction at the syntactic level.
    We
    study the complexity of grammar induction, and show that for a variety
    of type systems, including pivotal and compact closed
    categories, the grammar induction problem is NP-complete.
     Our approach could be extended to linguistic
     type systems such as autonomous or bi-closed categories.
\end{abstract}

\section{Introduction}

\subsection{Overview}

Category theoretic approaches to linguistics are flourishing. They provide
a convenient abstract framework for both syntax and semantics \cite{coecke2013lambek}, and
these insights enable some progress on natural language processing
tasks \cite{grefenstette2011experimental}.
This framework is flexible, because it allows for different types of grammars,
such as the Syntactic Calculus of Lambek \cite{lambek1968deductive} or Compact Bilinear Logic
\cite{preller2007free}, also known as pregroups \cite{lambek2008pregroup}.
It also allows for different kinds of
compositional semantics, which can be distributional \cite{coecke2013lambek},
Montagovian and extensional \cite{preller2005category}, Montagovian and intensional \cite{delpeuch2014natural},
or even hybrid models \cite{preller2011semantic}.
But whatever the syntax or the semantics are, these approaches rely on a functor
from the syntactic category to the semantic category to give meaning to a sentence.
\vspace{-0.3cm}
\begin{figure}[H]
    \centering
\begin{tabular}{c c c}
    $\mathcal{S}$ & $\xrightarrow{\text{\hspace{1cm}}F\text{\hspace{1cm}}}$ & $\cC$ \\
    syntax & & semantics
\end{tabular}
\end{figure}
\vspace{-0.3cm}

\noindent We propose to study the complexity of lifting the information flow at the semantic level
to a valid expression at the syntactic level. In a quantum setting, this could correspond
to representing a family of quantum circuits as (planar) string diagrams, for instance.
In a linguistic framework, this is the task of grammar induction.
Given a set of example sentences belonging to a language, the problem is
to infer a grammar of this language. Originally motivated
by the study of language acquisition by children \cite{pinker2009language}, this task has been widely
investigated in the field of formal languages
\cite{delahiguera2005bibliographical}. If the example sentences
are just raw strings, the problem is known to be intractable for most
expressive classes of grammars \cite{gold1967language}. Hence variations have been introduced,
one of them consisting in adding some semantic information about the words
in the example sentences.
In a categorical framework, words are given syntactic types, which are objects in a monoidal
category. The semantic type of a word is the image of this syntactic type under a monoidal functor
to the semantic category. The categories we will use are defined in Section~\ref{sec:catdef} and are summarised
in figure~\ref{fig:hierarchy}. Our results focus on the lower part of our hierarchy of categories,
which consists in quantum structures, whereas the linguistic type systems are higher up in the hierarchy.

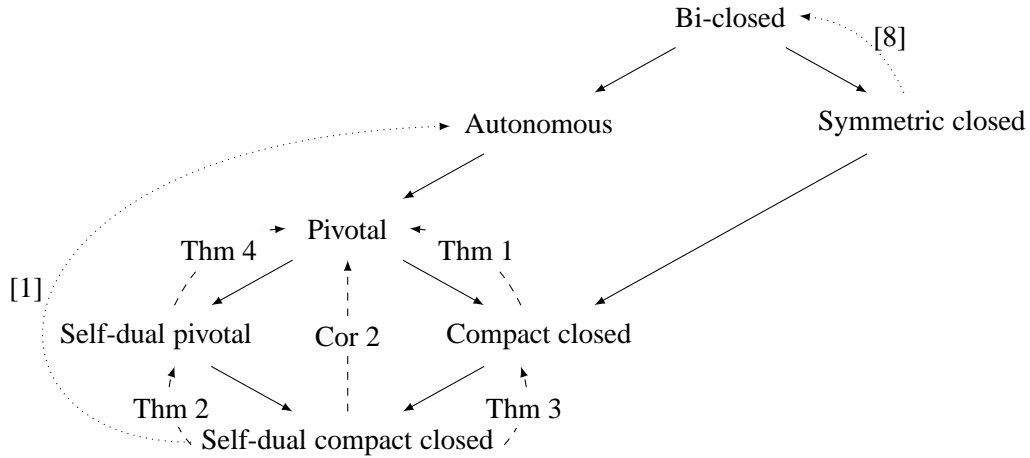
\begin{figure}[H]
    \centering
\begin{tikzpicture}[every node/.style={node distance=2cm}]

    \tikzstyle{s}=[xscale=1.8,rectangle,minimum height=0.8cm]

    \node[s,minimum width=1cm] (lambek) {};
    \node at (lambek) {Bi-closed};
\node[s,below right of=lambek,minimum width=1cm] (simple) {};
\node at (simple) {Symmetric closed};
\node[s,below left of=lambek,minimum width=1.3cm] (pregroup) {};
\node at (pregroup) {Autonomous};
\node[s,below left of=pregroup,minimum width=0.9cm] (group) {};
\node at (group) {Pivotal};
\node[s,below left of=group,minimum width=1.2cm] (monoid) {};
\node at (monoid) {Self-dual pivotal};
\node[s,below right of=monoid,minimum width=2.3cm] (commonoid) {};
\node at (commonoid) {Self-dual compact closed};
\node[s,below right of=group,minimum width=1.3cm] (abelian) {};
\node at (abelian) {Compact closed};

\draw[-latex] (lambek) -- (pregroup);
\draw[-latex] (lambek) -- (simple);
\draw[-latex] (simple) -- (abelian);
\draw[-latex] (pregroup) -- (group);
\draw[-latex] (group) -- (monoid);
\draw[-latex] (monoid) -- (commonoid);
\draw[-latex] (group) -- (abelian);
\draw[-latex] (abelian) -- (commonoid);

\tikzstyle{comp}=[dashed,-latex]
\tikzstyle{refs}=[dotted,-latex]

\draw[comp,bend left] (monoid) edge node[fill=white] {Thm~\ref{thm:monoid-group}} (group);
\draw[comp,bend right] (commonoid) edge node[fill=white] {Thm~\ref{thm:commonoid-abelian}} (abelian);
\draw[refs] (commonoid) .. controls ($(commonoid)+(-5.2,0)$) and ($(commonoid)+(-5,4)$) .. node[left] {\cite{bechet2007learnability}} (pregroup);

\draw[refs,bend right] (simple) edge node[right] {\cite{dudau-sofronie2001logic}} (lambek);

\draw[comp,bend right] (abelian) edge node[fill=white] {Thm~\ref{thm:abelian-group}} (group);
\draw[comp, bend left] (commonoid) edge node[fill=white] {Thm~\ref{thm:commonoid-monoid}} (monoid);
\draw[comp] (commonoid) edge node[fill=white] {Cor~\ref{coro:commonoid-group}} (group);

\end{tikzpicture}

Plain lines are functors, dashed lines are complexity results and dotted lines are existing algorithms.
\caption{A hierarchy of type systems}
\label{fig:hierarchy}
\end{figure}

Since the grammatical correctness of a sentence is witnessed by
an arrow from the product of its syntactic types to $S$ (the type of a sentence), the problem of grammar
induction can be seen as \emph{lifting} an arrow from the semantic category to the syntactic category, as
we will see in Section~\ref{sec:probdef}.

It turns out that many instances of this problem are \emph{hard}, in the sense of computational
complexity theory. This is mainly because we require that the syntactic type assigned to
each word remains consistent among all the example sentences. This creates global
constraints which restrict the solutions of the inference problem. In Section~\ref{sec:thms},
we use this fact to reduce NP-complete problems to our grammar learning problem.

\subsection{An example}

Suppose we use a compact closed category for the semantics and a pivotal category
for the syntax. We have to infer the possible syntactic types $t_i$ based on their
images $F(t_i)$, where $F$ is the canonical monoidal functor from the free pivotal category to
the free compact closed category on a given set of generators. In the following expressions, the tensor product $\otimes$ is implicit.
\begin{align}
    & \text{Syntax}  & & t_1 & & t_2 & & t_3 & & t_4 & & \rightarrow S \\
    & \text{Semantics} & & A B C & & B^* A^* A & & C^* A^* A & & A^* S & &\rightarrow S
\end{align}
There are many different arrows of the required domain and codomain at the semantic level. One of them is

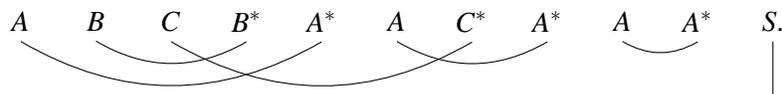
\vspace{-0.3cm}
\begin{figure}[H]
    \centering
\begin{tikzpicture}
    \foreach \t/\x in {$A$/1,$B$/2,$C$/3,$B^*$/4,$A^*$/5,$A$/6,$C^*$/7,$A^*$/8,$A$/9,$A^*$/10,$S$./11} {
        \node at (\x,0) (t\x) {\t};
    }
    \foreach \a/\b in {1/5,2/4,3/7,6/8,9/10} {
        \draw[bend right=30] (t\a.south) edge (t\b.south);
    }
    \draw (t11) -- (11,-1);
\end{tikzpicture}
\end{figure}

\vspace{-0.3cm}

As the only difference between a free compact closed category and a free pivotal category
is the symmetry, the problem bends down to finding a permutation of the basic types of
each $t_i$ such that the type reduction holds at the syntactic level. In other words, we have to
find a diagrammatic reduction without crossing, such as this one:

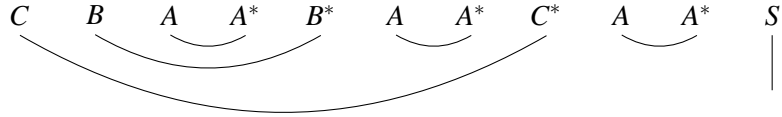
\vspace{-0.3cm}
\begin{figure}[H]
    \centering
\begin{tikzpicture}
    \foreach \t/\x in {$C$/1,$B$/2,$A$/3,$A^*$/4,$B^*$/5,$A$/6,$A^*$/7,$C^*$/8,$A$/9,$A^*$/10,$S$/11} {
        \node at (\x,0) (t\x) {\t};
    }
    \foreach \a/\b in {1/8,2/5,3/4,6/7,9/10} {
        \draw[bend right=30] (t\a.south) edge (t\b.south);
    }
    \draw (t11) -- (11,-1);
\end{tikzpicture}
\end{figure}

\vspace{-0.3cm}

In this particular example, one can see that it is necessary that $C$ occurs before $B$ in $t_1$.
We can add a second sentence:
\begin{align}
    & \text{Syntax}  & & t_1 & & t_5 & & t_6 & & t_7 & & \rightarrow S \\
    & \text{Semantics} & & A B C & & B^* C^* C & & A^* C^* C & & C^* S & &\rightarrow S
\end{align}
This examples forces $A$ to occur before $B$ in $t_1$.
Hence every solution of the learning problem made of these two sentences will be such that $C$ and $A$
occur before $B$ in $t_1$.
In Section~\ref{sec:thms}, this technique enables us to reduce the problem of betweenness~\cite{guttmann2006variations}
to our grammar lifting problem. This problem is known to be NP-complete.

\section{A grammar hierarchy}

\subsection{Monoidal categories as type systems}
\label{sec:typesystem}

We define how monoidal categories can be used as type systems. Both the syntactic and the semantic
categories will be seen as type systems in our induction problem.

\begin{defi}
    A \textbf{type system} $(\mathcal{C},S)$ is a strict monoidal category $\mathcal{C}$ with a distinguished object $S$
    in $\mathcal{C}$.
\end{defi}

\noindent When the object $S$ is clear from the context, the type system is simply noted $\mathcal{C}$.
The objects of this category will be used to denote types. We require the category to be monoidal, so that
we can define the sentence type as the product of the types of its words.
The distinguished object will play the role of the type for a grammatical sentence.
The arrows in the category play the role of reductions: $A$ reduces to $B$ when $\mathcal{C}(A,B)$ is not empty.

The type systems we will consider are monoidal categories with some additional structure
(which will be detailed in section~\ref{sec:catdef}),
and freely generated by a basic category, whose objects are called \textbf{basic types} and morphisms are
understood as subtyping relations: there is a morphism between two basic types $A$ and $B$ when $A$ is a subtype
of $B$.

\begin{defi}
    A \textbf{lexicon} $l$ over a set of words $W$ and a type system $(\mathcal{C},S)$ is a
    function $l : W \rightarrow \mathcal{C}$.
\end{defi}
\noindent Although it is interesting to consider the case where
multiple types can be assigned to a single word, the previous definition restricts
our lexicons to one type per word. We restruct ourselves to rigid grammars,
according to the terminology of \cite{bechet2007learnability}.

\begin{defi}
    A sequence of words $w_1, \dots, w_n \in W$ is \textbf{grammatical} for a lexicon $l$
    when $\mathcal{C}(l(w_1) \otimes \dots \otimes l(w_n), S)$ is not empty.
\end{defi}
\noindent In this definition, $S$ is the distinguished type of the underlying type system.

\begin{defi}
    A \textbf{functor} of type systems from $(\mathcal{C}_1,S_1)$ to $(\mathcal{C}_2,S_2)$
    is a functor of monoidal categories $F : \mathcal{C}_1 \rightarrow \mathcal{C}_2$ such that
    $F(S_1) = S_2$.
\end{defi}
\noindent From this definition, the following property follows immediately:
\begin{prop}
    Let $F : \mathcal{T}_1 \rightarrow \mathcal{T}_2$ be a functor of type systems.
    If a sentence $w_1, \dots, w_n$ is grammatical for the lexicon $\mathcal{L}_1$ over $\mathcal{T}_1$,
    then it is grammatical for the lexicon $F \circ \mathcal{L}_1$ over $\mathcal{T}_2$.
\end{prop}
\noindent This property expresses that if a sentence is correct at the syntactic level, then
there is a valid reduction at the semantic level.

\subsection{Various structures in monoidal categories}
\label{sec:catdef}
We now move on to the definition of the categories involved in the hierarchy of figure~\ref{fig:hierarchy}.
\begin{defi}
    A \textbf{bi-closed category} is a monoidal category
    in which for all object $B$, the functor $\_ \otimes B$ has
    a right adjoint $\_ / B$ and the functor $B \otimes \_$ has
    a right adjoint $B \backslash \_$.
\end{defi}

\noindent In other words, this means that for every pair of objects $A$, $B$, we have
morphisms $\text{eval}^l_{A,B} : B \otimes (B \backslash A) \rightarrow A$
and $\text{eval}^r_{A,B} : (A / B) \otimes B \rightarrow A$ satisfying some coherence
equations, and similarly some morphisms $A \rightarrow (A \otimes B) / B$ and
$A \rightarrow B \backslash (B \otimes A)$. Type systems built on bi-closed
categories correspond to grammars defined in the Syntactic
Calculus of Lambek.

\begin{defi}
    An \textbf{autonomous category}\footnote{Some authors use the name \textbf{compact closed category} instead,
        but this term has been used for both symmetric and planar categories. As we want to insist
        on the fact that these categories are not symmetric (contrary to some other
    ones in this article), we follow the terminology of \cite{selinger2011survey}.}
    is a monoidal category where for each object $A$,
    there are two objects, the left ($A^l$) and right ($A^r$) adjoints, equipped with
    four morphisms $\epsilon^l_A : A^l \otimes A \rightarrow 1$,
    $\epsilon^r_A : A \otimes A^r \rightarrow 1$, $\eta^l_A : 1 \rightarrow A \otimes A^l$
    and $\eta^r_A : 1 \rightarrow A^r \otimes A$ satisfying the following equalities :
    \begin{align*}
        (\epsilon^r_A \otimes 1_A) \circ (1_A \otimes \eta^r_A) = 1_A & &
        (\epsilon^l_A \otimes 1_{A^l}) \circ (1_{A^l} \otimes \eta^l_A) = 1_{A^l} \\
        (1_A \otimes \epsilon^l_A) \circ (\eta^l_A \otimes 1_A) = 1_A & &
        (1_{A^r} \otimes \epsilon^r_A) \circ (\eta^r_A \otimes 1_{A^r}) = 1_{A^r}
    \end{align*}
\end{defi}

\noindent Type systems built on a free autonomous category define pregroup grammars. For instance, let $n$ be the type of a noun phrase and $s$ be the distinguished type 
of a sentence. If we give the type $n$ to the words \emph{Mary} and \emph{John},
and the type $n^r \otimes s \otimes n^l$ to \emph{loves}, the sentence
\emph{Mary loves John} has the type $n \otimes n^r \otimes s \otimes n^l \otimes n$.
This type reduces to $s$ through the morphism
$$(\epsilon^r_n \otimes 1_s \otimes \epsilon^l_n) : n \otimes n^r \otimes s \otimes n^l \otimes n \rightarrow s$$
See \cite{lambek2008pregroup} for a linguistic presentation of pregroup grammars and \cite{preller2007free}
for the links with category theory.

The distinction between $n^l$ and $n^r$ is important at a syntactical level to reject
ill-formed sentences. For instance, we can give the type $s^r \otimes s$ to adverbs placed at the end of a sentence.
If $s^l = s^r$, then the type $s^r \otimes s = s^l \otimes s$ reduces to $1$ through $\epsilon^l_s$, hence
the adverb can be written at any place in the sentence, which does not reflect the usual rules of grammar.
As one can show that for any object $n$, $n^{rl} \simeq n \simeq n^{lr}$, the iterated adjoints of a type $n$
are of the form

$$\dots, n^{lll}, n^{ll}, n^{l}, n, n^{r}, n^{rr}, n^{rrr} \dots$$
so we can write $n^{ll} = n^{-2}, n^{l} = n^{-1}, n = n^{0}, n^{r} = n^{1}, n^{rr} = n^{2}$, and so on.

However, it makes sense to drop the distinction between left and right adjoints at the semantic level:
in terms of flow of information, an adjoint is just something that can consume a resource, no matter whether
it comes from the left or the right side.

\begin{defi}
    A \textbf{pivotal category} is an autonomous category with a monoidal natural isomorphism between $A^r$ and $A^l$.
    We set $A^* = A^l$.
\end{defi}

\noindent Pivotal categories correspond to groups, in the sense that in a free pivotal category,
two objects have an arrow between them if and only if they are equal in the corresponding
free group (where $^*$ plays the role of the inverse, hence $^*$ will be sometimes noted $^{-1}$).

The canonical morphism between the free pregroup and the free group is defined by
$$h : t_1^{e_1} \otimes \dots \otimes t_n^{e_n} \mapsto t_1^{(-1)^{e_1}} \otimes \dots \otimes t_n^{(-1)^{e_n}}$$
where $t_1^{e_1} \otimes \dots \otimes t_n^{e_n}$ is the canonical form of a pregroup element.

\begin{defi}
    A \textbf{compact closed category} is an autonomous category which is symmetric, i.e.
    for each objects $A$ and $B$ there is a monoidal natural isomorphism $s_{A,B} : A \otimes B \rightarrow B \otimes A$
    such that $s_{A,B}^{-1} = s_{B,A}$.
\end{defi}

\noindent For instance, the category of finite-dimensional vector spaces is compact closed.
One can wonder why we introduced the isomorphism $A^l \simeq A^r$ before adding the symmetries $s_{A,B}$.
The following fact explains our choice.
\begin{prop}
    Compact closed categories are pivotal.
\end{prop}
\noindent This property is well known (it is stated in \cite{coecke2013lambek}, and implicitly in \cite{selinger2011survey})
but I have never seen a proof of it.
\begin{proof}
    Let $\phi_A$ and $\psi_A$ be the following morphisms :
    
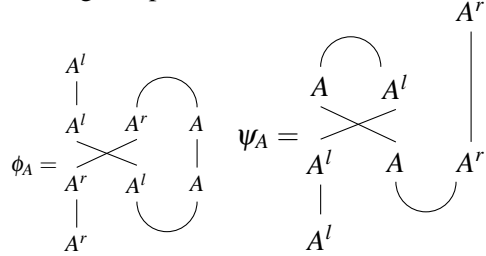
\vspace{-0.9cm}
    \begin{figure}[H]
    \centering
    \subfloat{
\begin{tikzpicture}[every node/.style={node distance=1cm,scale=0.8},scale=0.8]

    \node (Al0) {$A^l$};
    \node[below of=Al0] (Al1) {$A^l$};
    \node[below of=Al1] (Ar2) {$A^r$};
    \node[below of=Ar2] (Ar3) {$A^r$};
    \node[right of=Al1] (Ar1) {$A^r$};
    \node[right of=Ar1] (A1) {$A$};
    \node[right of=Ar2] (Al2) {$A^l$};
    \node[right of=Al2] (A2) {$A$};

    \draw (Al0) -- (Al1);
    \draw (Al1.south) -- (Al2.north);
    \draw (Ar1.south) -- (Ar2.north);
    \draw (A1) -- (A2);
    \draw (Ar2) -- (Ar3);
    \draw (Ar1.north) arc (180:0:0.5);
    \draw (Al2.south) arc (-180:0:0.5);

    \node[below left of=Al1] (phi) {$\phi_A =$};

\end{tikzpicture}
}
\subfloat{
\begin{tikzpicture}[every node/.style={node distance=1cm},scale=0.8]

    \node (A1) {$A$};
    \node[right of=A1] (Al1) {$A^l$};
    \node[right of=Al1] (Ar1) {};
    \node[above of=Ar1] (Ar0) {$A^r$};
    \node[below of=A1] (Al2) {$A^l$};
    \node[below of=Al1] (A2) {$A$};
    \node[right of=A2] (Ar2) {$A^r$};
    \node[below of=Al2] (Al3) {$A^l$};

    \draw (Ar0) -- (Ar2);
    \draw (A1.south) -- (A2.north);
    \draw (Al1.south) -- (Al2.north);
    \draw (Al2) -- (Al3);
    \draw (A1.north) arc (180:0:0.5);
    \draw (A2.south) arc (-180:0:0.5);

    \node[below left of=A1] (psi) {$\psi_A =$};

\end{tikzpicture}

}
\end{figure}

    \vspace{-0.3cm}
    We have $\psi_A \circ \phi_A = 1_{A^l}$ and $\phi_A \circ \psi_A = 1_{A^r}$.
    By symmetry, let us show the first equality only.

    
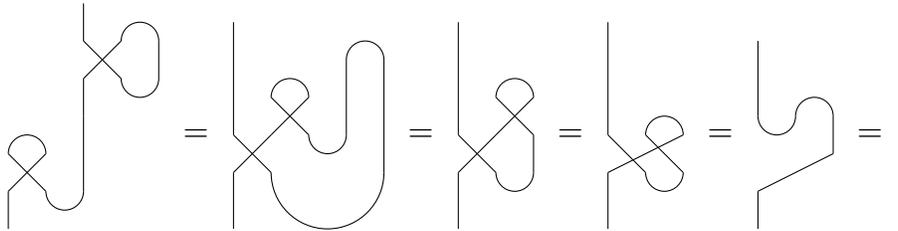
\begin{figure}[H]
    \centering
    \subfloat{
\begin{tikzpicture}[every node/.style={node distance=1cm,scale=0.5},scale=0.5]

    \node (Al0) {};
    \node[below of=Al0] (Al1) {};
    \node[below of=Al1] (Ar2) {};
    \node[below of=Ar2] (Ar3) {};
    \node[right of=Al1] (Ar1) {};
    \node[right of=Ar1] (A1) {};
    \node[right of=Ar2] (Al2) {};
    \node[right of=Al2] (A2) {};

    \draw (Al0.center) -- (Al1.center);
    \draw (Al1.center) -- (Al2.center);
    \draw (Ar1.center) -- (Ar2.center);
    \draw (A1.center) -- (A2.center);
    \draw (Ar2.center) -- (Ar3.center);
    \draw (Ar1.center) arc (180:0:0.5);
    \draw (Al2.center) arc (-180:0:0.5);

    \node[below of=Ar3] (Br1) {};
    \node[left of=Br1] (Bl1) {};
    \node[left of=Bl1] (B1) {};
    \node[below of=B1] (Bl2) {};
    \node[below of=Bl1] (B2) {};
    \node[right of=B2] (Br2) {};
    \node[below of=Bl2] (Bl3) {};

    \draw (Ar3.center) -- (Br2.center);
    \draw (B1.center) -- (B2.center);
    \draw (Bl1.center) -- (Bl2.center);
    \draw (Bl2.center) -- (Bl3.center);
    \draw (B1.center) arc (180:0:0.5);
    \draw (B2.center) arc (-180:0:0.5);

\end{tikzpicture}
}
\subfloat{
\begin{tikzpicture}[every node/.style={node distance=1cm,scale=0.5},scale=0.5]

    \node (A1) {};
    \node[below of=A1] (B1) {};
    \node[right of=B1,node distance=3cm] (B2) {};
    \node[right of=B2] (B3) {};
    \node[below of=B1] (C1) {};
    \node[right of=C1] (C2) {};
    \node[right of=C2] (C3) {};
    \node[below of=C1] (D1) {};
    \node[right of=D1] (D2) {};
    \node[right of=D2] (D3) {};
    \node[right of=D3] (D4) {};
    \node[below of=D1] (E1) {};
    \node[right of=E1] (E2) {};
    \node[right of=E2,node distance=3cm] (E3) {};
    \node[below of=E1,node distance=1.5cm] (F1) {};

    \draw (A1.center) -- (D1.center);
    \draw (D1.center) -- (E2.center);
    \draw (D2.center) -- (E1.center);
    \draw (C2.center) -- (D3.center);
    \draw (C3.center) -- (D2.center);
    \draw (B2.center) -- (D4.center);
    \draw (B3.center) -- (E3.center);
    \draw (E1.center) -- (F1.center);
    \draw (C2.center) arc (180:0:0.5);
    \draw (E2.center) arc (-180:0:1.5);
    \draw (D3.center) arc (-180:0:0.5);
    \draw (B2.center) arc (180:0:0.5);

    \node[left of=D1] (equal) {{\Huge $=$}};

\end{tikzpicture}

}
\subfloat{
\begin{tikzpicture}[every node/.style={node distance=1cm,scale=0.5},scale=0.5]

    \node (A1) {};
    \node[below of=A1] (B1) {};
    \node[below of=B1] (C1) {};
    \node[right of=C1] (C2) {};
    \node[right of=C2] (C3) {};
    \node[below of=C1] (D1) {};
    \node[right of=D1] (D2) {};
    \node[right of=D2] (D3) {};
    \node[below of=D1] (E1) {};
    \node[right of=E1] (E2) {};
    \node[right of=E2] (E3) {};
    \node[below of=E1,node distance=1.5cm] (F1) {};

    \draw (A1.center) -- (D1.center);
    \draw (D1.center) -- (E2.center);
    \draw (D2.center) -- (E1.center);
    \draw (C2.center) -- (D3.center);
    \draw (C3.center) -- (D2.center);
    \draw (E1.center) -- (F1.center);
    \draw (C2.center) arc (180:0:0.5);
    \draw (E2.center) arc (-180:0:0.5);
    \draw (E3.center) -- (D3.center);

    \node[left of=D1] (equal) {\Huge $=$};

\end{tikzpicture}

}
\subfloat{
\begin{tikzpicture}[every node/.style={node distance=1cm,scale=0.5},scale=0.5]

    \node (A1) {};
    \node[below of=A1,node distance=3cm] (B1) {};
    \node[below of=B1] (C1) {};
    \node[right of=B1] (B2) {};
    \node[right of=B2] (B3) {};
    \node[right of=C1] (C2) {};
    \node[right of=C2] (C3) {};
    \node[below of=C1] (D1) {};
    \node[below of=E1,node distance=1.5cm] (F1) {};

    \draw (A1.center) -- (B1.center);
    \draw (B1.center) -- (C2.center);
    \draw (B3.center) -- (C1.center);
    \draw (C1.center) -- (F1.center);
    \draw (B2.center) -- (C3.center);

    \draw (B2.center) arc (180:0:0.5);
    \draw (C2.center) arc (-180:0:0.5);

    \node[left of=B1] (equal) {\Huge $=$};
\end{tikzpicture}

}
\subfloat{
\begin{tikzpicture}[every node/.style={node distance=1cm,scale=0.5},scale=0.5]

    \node (A1) {};
    \node[below of=A1,node distance=2cm] (B1) {};
    \node[below of=B1] (C1) {};
    \node[right of=B1] (B2) {};
    \node[right of=B2] (B3) {};
    \node[right of=C1] (C2) {};
    \node[right of=C2] (C3) {};
    \node[below of=C1] (D1) {};
    \node[right of=D1] (D2) {};
    \node[right of=D2] (D3) {};
    \node[below of=D1] (E1) {};

    \draw (A1.center) -- (B1.center);
    \draw (B1.center) arc (-180:0:0.5);
    \draw (B2.center) arc (180:0:0.5);
    \draw (B3.center) -- (C3.center);
    \draw (C3.center) -- (D1.center);
    \draw (D1.center) -- (E1.center);

    \node at ($(C1)+(-1,0.5)$) (equal) {\Huge $=$};
\end{tikzpicture}

}
\subfloat{
\begin{tikzpicture}[every node/.style={node distance=1cm,scale=0.5},scale=0.5]
    \node (A1) {};
    \node[below of=A1,node distance=2.5cm] (B1) {};
    \node[below of=B1,node distance=2.5cm] (C1) {};
    
    \draw (A1.center) -- (C1.center);

    \node[left of=B1] (equal) {\Huge $=$};
\end{tikzpicture}
}

\end{figure}

    Moreover, one can check with similar techniques that this isomorphism is monoidal and natural.
\end{proof}

\begin{defi}
    A \textbf{self-dual compact closed category} is a compact closed category with
    a family of isomorphisms $h_A : A \rightarrow A^*$.
\end{defi}

\noindent Self-dual compact closed categories have been studied in detail by Selinger
in \cite{selinger2010autonomous}.
The definition we adopt here corresponds to his first option, namely self-duality without coherence.
As a finite-dimensional vector space is isomorphic to its dual, the category of
finite-dimensional vector spaces is self-dual.
This category has been widely used as the underlying semantic category
for models of meaning, such as in 
\cite{coecke2013lambek}, \cite{preller2011semantic} or \cite{delpeuch2014natural}.
The objects in this category have also been used in \cite{bechet2007learnability} as
semantic types in a learning task. However, they did not introduce a whole typing
system at the semantic level, as they had no notion of reduction on semantic types.

We have introduced the commutativity first and then the isomorphism between $A$ and $A^*$.
It is possible to swap these properties, although it requires to be more careful: 
\begin{defi}
    A free \textbf{self-dual pivotal category} is the free pivotal category generated by
    a category $\mathcal{C}$ where for each object $A \in \mathcal{C}$, $A \simeq A^*$.
\end{defi}

\noindent A self-dual pivotal category models a rewriting system where any two identical adjacent letters
cancel.

It is important to notice that we require that $A \simeq A^*$ only for basic objects.
If this were true for all objects, then as noted by Selinger \cite{selinger2011survey}, we
would get the following isomorphism
$$A \otimes B \simeq (A \otimes B)^* \simeq B^* \otimes A^* \simeq B \otimes A$$.
This isomorphism is not a symmetry in general but would have the same effects on our type system.

A widespread category for semantic types in the linguistic literature
is the free symmetric monoidal closed category.
It has been used, among others, in \cite{buszkowski1984note} and \cite{dudau2003learnable}.
\begin{defi}
    A \textbf{symmetric closed category} is
    a symmetric bi-closed category.
    For all objects $A$ and $B$, $B \backslash A \simeq A / B$,
    so we note $A | B = A / B$.
\end{defi}

\noindent The objects of this category can be thought of simple types for the
simply-typed $\lambda$-calculus with pairs. The object $A | B$ plays the role of
the type $B \rightarrow A$
and we have a morphism $\text{eval}_{A,B} : (A | B) \otimes B \rightarrow A$
satisfying the required coherence conditions.

\section{Functional types}
\subsection{Restricting the set of possible types}

Not all types are likely to be used in a type-logical grammar.
We expect types to be functional, i.e. to be built using only
abstractions, the operations $\backslash$ and $/$.

For instance, the type $n \otimes s \otimes n$ belongs to the free pregroup generated
by $n$ and $s$, but cannot be constructed by iterated abstractions.
The type $n^r \otimes s \otimes n^l$ however can be constructed as $n \backslash (s / n)$ or
$(n \backslash s) / n$.

\begin{defi}
    Let $\mathcal{L}$ be the free bi-closed monoidal category. The set $P \subset Ob(L)$ is the closure
    by $/$ and $\backslash$ of the set of basic types.
    Given a type system $(\mathcal{C},S)$ and a bi-closed functor $F : \mathcal{L} \rightarrow \mathcal{C}$
    the set of \textbf{functional types} in $\cC$ is $F(P)$.
\end{defi}

Restricting our search of types to this form of type reduces our search space.
This restriction makes sense because these types are more likely to be relevant
from a linguistic point of view.
For instance, \cite{lambek2008pregroup} builds a fairly advanced grammar of English and he uses
only functional types in his grammar, while not mentioning this constraint at all.

\subsection{Properties of functional types}

The generative power of pregroup grammars is not reduced when we require
functional types: the proof given in \cite{buszkowski2008pregroup} that every $\epsilon$-free
context free grammar is weakly equivalent to a pregroup grammar uses only functional types.

For group grammars (i.e. type systems built on pivotal categories), restricting the assignments
to functional types does not harm the expressiveness either, as it is enough to multiply by $a^{-1}a$
the types that are not functional to get an equivalent grammar with functional types only. This remark
will be made clear by the following proposition, which characterises functional types in pivotal categories.

\begin{prop}
    In a pivotal category, functional types are exactly those which are either
    \begin{itemize}
        \item basic types (generators of the free autonomous category), or
        \item products of basic types with exponents $t_1^{e_1} \otimes \dots \otimes t_n^{e_n}$,
            where at least one $e_i$ is $-1$ and at least one $e_i$ is $+1$.
    \end{itemize}
\end{prop}

\begin{proof}
    By induction on a functional Lambek type $t$, let us show that $F(t)$ satisfies the characterization above.
    If $t = a$, a basic type, then $F(t) = a$, falling into the first option.
    If $t = u / v$, then $F(t) = F(u)F(v)^{-1}$. By induction, there is a basic type occurring with a $+1$ exponent
    in $F(u)$, so it occurs again with the same exponent in $F(t)$. Similarly, there is a basic type occurring with
    a $+1$ exponent in $F(v)$, so it occurs with a $-1$ exponent in $F(t)$.

    Conversely, let us show by induction on the length of a group type $t = t_1^{e_1} \otimes \dots \otimes t_n^{e_n}$
    satisfying the characterization that it is the image of a functional Lambek type.
    If $n = 1$, then $t = a$ where $a$ is a basic type, so $F(a) = t$.
    If $n > 1$, there are several cases:
    \begin{itemize}
        \item $e_n = -1$ and $e_1, \dots, e_{n-1}$ satisfies the characterization. Then by induction
            we can find a functional Lambek type $u$ such that $F(u) = t_1^{e_1} \otimes \dots \otimes t_{n-1}^{e_{n-1}}$
            and hence $F(u / t_n) = F(u)F(t_n)^{-1} = t$.
        \item $e_n = -1$ and $e_1, \dots, e_{n-1} = +1$ : then $(-e_2), \dots, (-e_n)$ satisfies the
            characterization and hence there is a functional $u$ such that $F(u)^{-1} = t_2^{e_2} \otimes \dots
            \otimes t_n^{e_n}$, hence $F(t_1 / u) = t$.
        \item if $e_n = +1$ and $(-e_1), \dots, (-e_{n-1})$ satisfies the characterization. Then
              by induction we can find a functional Lambek type $u$ such that $F(u)^{-1} = t_1^{e_1} \otimes \dots \otimes t_{n-1}^{e_{n-1}}$ and hence $F(u \backslash t_n) = F(u)^{-1}F(t_n) = t$
          \item if $e_n = +1$ and $e_1, \dots, e_{n-1} = -1$: then $e_2, \dots, e_n$ satisfies the characterization and hence there is a functional $u$ such that $F(u) = t_2^{e_2} \otimes \dots t_n^{e_n}$, hence $F(t_1 \backslash u) = F(t_1)^{-1}F(u) = t$.
    \end{itemize}
    This completes the proof.
\end{proof}

\begin{coro}
    In a compact closed category, the characterization of functional types is the same.
    In a self-dual compact closed category, every type but $1$ is functional.
\end{coro}

\section{Complexity of the grammar induction problem}
\label{sec:thms}

\subsection{Definition of the problem}
\label{sec:probdef}

We study the complexity of learning syntactic types based on positive samples (i.e. a set of grammatical
sentences) with semantic types. Each word occurrence in the samples comes with a semantic type.
The nature of the syntactic and semantic types depends on the problem.

\begin{defi}
    A \textbf{training sample} for a type system $(\mathcal{C},S)$ and a finite set of variables $V$
    is a finite set of sentences, where each sentence is a finite sequence of the form
    $(v_1,t_1), \dots, (v_n,t_n)$, where $v_i \in V$ and $t_i \in \mathcal{C}$ is functional,
    and such that all the sentences are grammatical for their respective type assignment.
\end{defi}

\noindent Note that we do not require that a variable is always paired with a single type.
The type of the word can depend on the context in which it appears.

In the following sections, we study the complexity of inducing a grammar, given 
a finite training sample. First we give a definition of the problem.

\begin{defi}
    Let $(\mathcal{C},S)$ be the syntactic type system, $(\mathcal{C}',S')$ be the semantic type system
    and $F$ be a morphism from $(\mathcal{C},S)$ to $(\mathcal{C}',S')$.
    We call \textbf{grammar induction} the problem of, given a training sample $T$
    for the type system $(\mathcal{C}',S')$,
    find a lexicon $h : V \times Ob(\mathcal{C}') \rightarrow \mathcal{C}$
    such that
    $$\text{for all pair } (v_i,t_i) \in T, F(h(v_i,t_i)) \simeq t_i\text{ and $h(v_i,t_i)$ is functional}$$
    $$\text{and for all sentence }(v_1,t_1), \dots, (v_n,t_n) \in T, \mathcal{C}(h(v_1,t_1) \otimes \dots \otimes h(v_n,t_n), S) \text{ is not empty}$$
\end{defi}

\noindent In other words, the problem is to find functional syntactic types that are compatible with the semantic types and
all the sentences are grammatical at the syntactic level.
Note that we require that each pair (variable, semantic type) is associated to an unique
syntactic type, following \cite{bechet2007learnability}. Without this restriction, the problem
is trivial as the syntactic types can be chosen independently for each sentence.

\subsection{Learning pivotal categories from compact closed categories}

\begin{thm} \label{thm:abelian-group}
    Type inference from a compact closed category to a pivotal category is NP-complete.
\end{thm}

\begin{proof}
    We give a reduction of the betweenness problem \cite{guttmann2006variations}
    to our grammar induction problem.
    The betweenness problem is as follows. Given a finite set $A$ and a set of triples $C \subset A^3$,
    the problem is to find a total ordering of $A$ such that for each $(a,b,c) \in C$, either $a < b < c$ or $c < b < a$.
    This problem is NP-complete \cite{guttmann2006variations}.

    The compact closed category we will consider contains the objects $a$ for each $a \in A$
    and $d_{a,b,c}$ for each $(a,b,c) \in C$, with the following reduction between
    basic types: $a \rightarrow d_{a,b,c}$ and $c \rightarrow d_{a,b,c}$.
    We set $y = \prod_{x \in A} x$. The preimage of this type will define the total order satisfying 
    the constraints induced by the sentences.
    For each triple $(a,b,c) \in C$, we define the following compact closed types:
    \begin{align*}
        w = \prod_{x \in A \backslash \{a,b,c\}} x & & c_1 = d_{a,b,c}^{-1} w^{-1} w & & c_2 = b^{-1} w^{-1} w 
    \end{align*}
    and add the following sentence to the training sample:
    $$(Y,y) (W_{a,b,c,1},c_1) (W_{a,b,c,2},c_2) (W_{a,b,c,3},c_1) (W_{a,b,c,4}, w^{-1})$$
    where the $W$ are words chosen to be different from any word previously seen.
     
    This reduction is polynomial.
    Let us show that this grammar induction problem has a solution if and only if
    the corresponding betweenness problem has a solution.
    If there is a total ordering $<$ of $A$ satisfying the constraints, let $A = \{ x_1, \dots, x_n \}$
    where $x_1 < \dots < x_n$.
    One can check that with the following preimages, the sample is grammatical in the pivotal category:
    \begin{itemize}
        \item The type of $Y$ becomes $y' = \prod_{i = 1}^n x_i$.
        \item For each $(a,b,c) \in C$, let $p$,$q$,$r$ and $s$ be such that
            \begin{align*}
                y' = p \cdot a \cdot q \cdot b \cdot r \cdot c \cdot s & & \text{ or } & &
                y' = p \cdot c \cdot q \cdot b \cdot r \cdot a \cdot s
            \end{align*}
            (where $p$, $q$, $r$ and $s$ are possibly equal to $1$). $y'$ reduces to $p \cdot d_{a,b,c} \cdot q \cdot b \cdot r \cdot d_{a,b,c} \cdot s$.

    P possible type assignment for the $W$ is:
    \begin{align*}
        W_{a,b,c,1} : s^{-1} d_{a,b,c}^{-1} s p q r (p q r)^{-1} & & W_{a,b,c,2} : (r s)^{-1} b^{-1} r s (p q) (p q)^{-1}
        & & W_{a,b,c,3} : (q r s)^{-1} d_{a,b,c}^{-1} (q r s) p p^{-1}
    \end{align*}
    One can check that the image of this assignment is equal to the assignment from the training sample
    and that it makes the sentences grammatical in the pivotal category.
    \end{itemize}
    Conversely, if there exists a pivotal type assignment, then as the type $b$ does not occur in
    the types assigned to $W_{a,b,c,1}$ and $W_{a,b,c,3}$, there is an $a$ or a $c$ on the right side of the occurrence
    of $b$, and similarly on the left side. But as there cannot be two occurrences of the same basic type in $y'$,
    we have either $a < b < c$ or $c < b < a$.

    Hence the problem is NP-hard. As one can check a solution in polynomial time, the problem is NP-complete.
\end{proof}

\subsection{Learning self-dual pivotal categories from self-dual compact closed categories}

Similarly, the previous proof can also be carried when $a \simeq a^{-1}$, giving
the following theorem:

\begin{thm} \label{thm:commonoid-monoid}
    Type inference from a self-dual compact closed category to a self-dual pivotal category is NP-complete.
\end{thm}

\subsection{Learning compact closed categories from self-dual compact closed categories}

The problem of grammar induction from a self-dual compact closed category to a
compact closed category bends down to assigning exponents to the types.
It can be reduced to an integer linear
programming problem where we are interested in nonnegative solutions only.
This problem is NP-complete
and we will show that the grammar induction problem itself is actually NP-complete.

\begin{thm} \label{thm:commonoid-abelian}
    Type inference from a self-dual compact closed category to a compact closed category is NP-complete.
\end{thm}

\noindent We give a polynomial reduction from 3-SAT to the problem of learning
symmetric pivotal types from self-adjoint (symmetric pivotal) types.
As 3-SAT is NP-complete \cite{karp1972reducibility}, and the learning problem is in NP,
this will complete the proof.
Let $\phi = c_1 \wedge \dots \wedge c_n$ be a conjunction of 3-clauses.
We adopt an approach similar to the strategy of~\cite{cohen2010viterbi}:
\begin{enumerate}[(i)]
  \item We replace each positive occurrence of a variable $x_i$ by
       a new variable $x_{i,1}$ and each negative occurrence by a different
       variable $x_{i,0}$. We add to the training sample one sentence,
       ensuring that all the clauses are satisfied (lemma~\ref{lemma:full-sentence}).
  \item We add one sentence for each pair of variables $x_{i,0},x_{i,1}$,
       ensuring that they are assigned opposite truth values (lemma~\ref{lemma:negation}).
\end{enumerate}
The full training example will hence encode satisfiability, completing the proof.

Let $c = x_{i_1,b_1} \vee x_{i_2,b_2} \vee x_{i_3,b_3}$ be a 3-clause,
where $x_{i,1}$ stands for $x_i$ and $x_{i,0}$ for $\neg x_i$.
For each $i$ and $b$ we define a self-adjoint type $v_{i,b} = z_{i,b} y_{i,b} y_{i,b}$.
Let $t(c) = d z_{i_1,b_1} z_{i_2,b_2} z_{i_3,b_3}$ be a self-adjoint type.
Our idea is that we will force $d$ to have a $-1$ exponent in the corresponding
group type, and hence this group type will be functional if and only if one of
the $z_{a,b}$ occurs with a $+1$ exponent. As a clause is true when at least
one of the literals it contains is true, this will encode satisfiability.

\begin{lemma} \label{lemma:full-sentence}
For each literal $x_i^b, i \in \{1, \dots p\}, b \in \{0,1\}$, let $n_{i,b}$
be the number of occurrences of $x_{i,b}$ in $\phi$.
The assignments making the following sentence grammatical
$$(S,s) \prod_{i=1}^n ( (C_i,t(c_i)) (D_i,d) ) \prod_{i=1}^p \prod_{b=0}^1 (X_i^b,v_{i,b})^{n_{i,b}}$$
are exactly those for which $x_{1,0}, x_{1,1}, x_{2,0}, x_{2,1}, \dots, x_{p,0}, x_{p,1}$
are assigned values making all the clauses true.
\end{lemma}

\begin{proof}
Let us show first that if $\phi$ is satisfiable, then the grammar induction problem has a solution.
Let $x_1 = a_1, \dots, x_p = a_p$ be a satisfying boolean assignment.
We give $X_i^b$ the group type $z_{i,b}^e y_{i,b}^{-1} y_{i,b}$, where
$e = -1$ if $b = a_i$ and $e = 1$ otherwise. This type is functional.

We give $C_i$ the type $c^{-1} z_{i_1,b_{i,1}}^{e_1} z_{i_2,b_{i,2}}^{e_2} z_{i_3,b_{i,3}}^{e_3}$
where $e_k = 1$ if $b_{i,k} = a_{i_k}$ and $e_k = -1$ otherwise. As the clause $c_i$ is satisfied,
there is at least one $k \in \{ 1, 2, 3 \}$ such that $b_{i,k} = a_{i_k}$, hence the type is functional.

Let us show that the sentence is grammatical. As the exponent of $z_{i,b}$ in the type assigned
to a clause only depends on $a_i$, there are $n_{i,b}$ occurrences of $z_{i,b}$, with the same exponent,
in $\prod_{j=1}^n (c_j,t_j)$. By construction, the exponent of $z_{i,b}$ is inversed in the type assigned
to $x_i^b$, and there are $n_{i,b}$ such occurrences in the sentence. Hence all the $z_{i,b}$ cancel.
The type $d$ assigned to $D_i$ cancels with $d^{-1}$ in the type assigned to $t(c_i)$, and
the $y_{i,b}$ cancel as well. Hence only $s$ remains: the sentence is grammatical.

Conversely, suppose there are functional group types $r_j$ preimage of $t(c_j)$ and $w_{i,b}$ preimage of $v_{i,b}$
such that the sentences are grammatical at the syntactic level.
Note that the types of the words $S$ and $D_i$ are basic types, so the only functional syntactic types
compatible with the learning problem are these basic types.
As the pregroup type $d$ occurs with exponent $+1$ $n$ times in the product due to the words $D_j$,
all the occurences of $d$ in the $r_j$ have the exponent $-1$, otherwise they would not cancel.
For each $r_j$, it is functional so one of the $z_{i,b}$ has exponent $+1$.
For each $x^b_i$, $z_{i,b}$ occurs $n_{i,b}$ times with the same exponent, thanks to $w_{i,b}$,
and $n_{i,b}$ other times in the clauses, so the exponent assigned to $z_{i,b}$ is the same in every $r_j$.
\end{proof}

\begin{lemma} \label{lemma:negation}
The assignments making the following sentence grammatical
$$(S,s) (X_i^0, v_{i,0}) (Z_i, z_{i,0} z_{i,1}) (X_i^1, v_{i,1})$$
are exactly those for which the $z_{i,0}$ and $z_{i,1}$ in $v_{i,0}$ and $v_{i,1}$
get opposite exponents.
\end{lemma}

\begin{proof}
One can check that the following assignment is valid, for $e = \pm 1$:
$$s \cdot z_{i,0}^e y_{i,0}^{-1} y_{i,0} \cdot z_{i,0}^{-e} z_{i,1}^{e} \cdot z_{i,1}^{-e} y_{i,1}^{-1} y_{i,1} \rightarrow s$$

Conversely, as the type assigned to $Z_i$ has to be functional, $z_{i,0}$ and $z_{i,1}$ get opposite exponents
in any solution of the grammar inference problem.
\end{proof}

\subsection{Learning pivotal categories from self-dual pivotal categories}

The construction of Theorem~\ref{thm:commonoid-abelian} can be adapted to work in non-commutative structures,
hence the following theorem:

\begin{thm} \label{thm:monoid-group}
    Type inference from self-dual pivotal categories to pivotal categories is NP-complete.
\end{thm}

\subsection{Composing complexity results}

Suppose we know the complexity of the grammar induction problem
between $\mathcal{C}_1$ and $\mathcal{C}_2$, and between $\mathcal{C}_2$
and $\mathcal{C}_3$. What can be said about grammar induction between
$\mathcal{C}_1$ and $\mathcal{C}_3$?

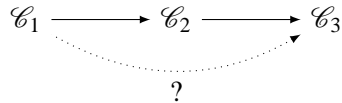
\begin{figure}[H]
    \centering
    \begin{tikzpicture}[every node/.style={node distance=2cm}]
    \node (c1) {$\mathcal{C}_1$};
    \node[right of=c1] (c2) {$\mathcal{C}_2$};
    \node[right of=c2] (c3) {$\mathcal{C}_3$};

    \draw[latex-] (c2) -- (c1);
    \draw[latex-] (c3) -- (c2);
    \draw[latex-,bend left,dotted] (c3) edge node[below] {?} (c1);
\end{tikzpicture}
\end{figure}

Given a syntactic category $\mathcal{C}$ and a semantic category $\mathcal{C}'$,
we introduce the notion of exact samples.

\begin{defi}
    A training sample is said \textbf{exact} for some syntactic type $t$ when
    it contains a word-type pair $(w,F(t))$ such that for all solutions $h$ of this training sample,
    $h(w,F(t)) = t$.

    We say that a grammar induction problem \textbf{has exact samples} when
    there exists exact samples for each syntactic type $t$.
\end{defi}

In other words, a grammar induction problem has exact samples when we can build sentences
forcing the preimage of a particular type.

\begin{lemma}
    If the grammar induction problem has exact samples,
    then for all finite set of syntactic types $T = \{ t_1, \dots, t_n \}$
    there exists a training sample which is exact for $t_1, \dots, t_n$.
\end{lemma}

\begin{proof}
    Take an exact training samples for each element of $T$. Make these
    training samples disjoint by ensuring that they use different words.
    The concatenation of these training samples satisfies the property claimed.
\end{proof}

\begin{lemma} \label{lemma:exact-comp}
    If grammar induction from $\cC_1$ to $\cC_2$ and from $\cC_2$ to $\cC_3$ has exact samples,
    then so does grammar induction from $\cC_1$ to $\cC_3$.
\end{lemma}

\begin{proof}
    Take an exact sample $S$ for $t$ from $\cC_2$ to $\cC_3$.
    For each type $t'$ occurring in $S$, take exact samples
    for $t'$ from $\cC_1$ to $\cC_2$. Concatenate these samples
    with the image of $S$ under the functor from $\cC_2$ to $\cC_1$.
\end{proof}

\begin{prop} \label{prop:exact-NP}
    If grammar induction from $\cC_1$ to $\cC_2$ has (polynomial) exact samples
    and grammar induction from $\cC_2$ to $\cC_3$ is NP-complete, then
    grammar induction from $\cC_1$ to $\cC_3$ is NP-complete.
\end{prop}
\begin{proof}
    Take an instance of SAT. It can be represented as an
    equivalent training sample from $\cC_2$ to $\cC_3$.
    Take the image of this training sample by the functor from $\cC_2$
    to $\cC_1$ and force this new sample to have the original preimages
    in $\cC_2$ by adding an exact sample. This problem has a solution
    if and only if the instance of SAT is satisfiable.
\end{proof}
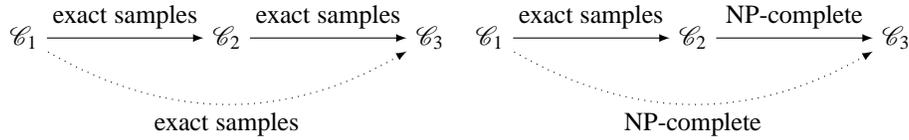
\vspace{-0.7cm}
\begin{figure}[H]
    \centering
    \subfloat{
    \centering
    \begin{tikzpicture}[every node/.style={node distance=3cm,scale=0.9},scale=0.9]
    \node (c1) {$\mathcal{C}_1$};
    \node[right of=c1] (c2) {$\mathcal{C}_2$};
    \node[right of=c2] (c3) {$\mathcal{C}_3$};

    \draw[latex-] (c2) edge node[above] {exact samples} (c1);
    \draw[latex-] (c3) edge node[above] {exact samples} (c2);
    \draw[latex-,bend left,dotted] (c3) edge node[below] {exact samples} (c1);
\end{tikzpicture}
}
\subfloat{

    \centering
    \begin{tikzpicture}[every node/.style={node distance=3cm,scale=0.9},scale=0.9]
    \node (c1) {$\mathcal{C}_1$};
    \node[right of=c1] (c2) {$\mathcal{C}_2$};
    \node[right of=c2] (c3) {$\mathcal{C}_3$};

    \draw[latex-] (c2) edge node[above] {exact samples} (c1);
    \draw[latex-] (c3) edge node[above] {NP-complete} (c2);
    \draw[latex-,bend left,dotted] (c3) edge node[below] {NP-complete} (c1);
\end{tikzpicture}
}
\end{figure}
\vspace{-0.3cm}
\begin{lemma} \label{lemma:exact-list}
The following grammar induction problems have exact samples:
\begin{itemize}
    \item from self-dual compact closed to self-dual pivotal
    \item from self-dual compact closed to compact closed
    \item from self-dual pivotal to pivotal
    \item from compact closed to pivotal
\end{itemize}
\end{lemma}
\begin{proof}
    Use the same techniques as the ones we developped for our reductions.
\end{proof}

\begin{coro} \label{coro:commonoid-group}
    Grammar induction from a self-dual compact closed category to a pivotal category is NP-complete.
\end{coro}
\begin{proof}
    Combine Lemma~\ref{lemma:exact-list} and Theorem~\ref{thm:abelian-group} with
    Proposition~\ref{prop:exact-NP}.
\end{proof}
\section{Future work}
A number of questions remain open. Being able to classify the complexity of 
the inference problem in the higher half of the hierarchy would enable us
to give complexity results on the problems studied in \cite{bechet2007learnability}
and \cite{dudau-sofronie2001logic}.

Another issue is the expressivity of the classes of grammars defined by the
categories in the lower half of the hierarchy. These grammars generate sub-classes
of the context-free grammars, but it would be interesting to relate these sub-classes
to known classes from the field of formal languages.

One could also use this framework to study inference problems in which the structure
of a parse is known, but the types are unknown. This notion of learning with structural
examples has been studied for the syntactic calculus \cite{kanazawa1996identification}.

\section*{Acknowledgements}

This work has been supported by the Quantum Group in the Department of Computer Science of the University of Oxford.
I am grateful to Thomas Bourgeat, Amar Hadzihasanovic, Anne Preller and Isabelle Tellier
for their help, and the reviewers for their useful comments.
Special thanks go to Jamie Vicary who helped me a lot by supervising and reviewing my work.

\vspace{-0.3cm}
\bibliographystyle{eptcs}
\bibliography{references}

\end{document}